\documentclass[letterpaper, 10 pt, conference]{ieeeconf}
\usepackage[T1]{fontenc}
\usepackage[latin9]{inputenc}
\usepackage{xcolor}
\usepackage{verbatim}
\usepackage{float}
\usepackage{amsmath}

\usepackage[T1]{fontenc}
\usepackage[latin9]{inputenc}
\usepackage{color}
\usepackage{verbatim}
\usepackage{float}
\usepackage{amsmath}

\usepackage{amsthm}
\usepackage{amssymb}
\usepackage{stmaryrd}
\usepackage{stackrel}
\usepackage{graphicx}
\usepackage{wasysym}
\usepackage{mathtools}
\usepackage{hyperref}
\usepackage{url}   
\usepackage{booktabs} 
\usepackage{nicefrac} 
\usepackage{authblk}
\usepackage{multirow}
\usepackage{xcolor}
\usepackage{microtype}
\usepackage{paralist}
\usepackage{todonotes}
\usepackage{subfig}
\usepackage{graphicx}
\usepackage{paralist}

\makeatletter

\floatstyle{ruled}

\usepackage[ruled,vlined,linesnumbered]{algorithm2e}
\usepackage{algpseudocode}

\theoremstyle{plain}
\newtheorem{thm}{\protect\theoremname}
\theoremstyle{definition}

\theoremstyle{definition}
\newtheorem{problem}{\protect\problemname}
\theoremstyle{plain}
\newtheorem{lem}{\protect\lemmaname}
\theoremstyle{definition}

\theoremstyle{remark}
\newtheorem{rem}{\protect\remarkname}
\theoremstyle{plain}

\newtheorem{definition}{Definition}

\usepackage{cite}
\usepackage{algpseudocode}
\usepackage{setspace}
\IEEEoverridecommandlockouts

\usepackage{float}
\usepackage{algpseudocode}
\usepackage{setspace}
\usepackage{color}
\usepackage{comment}
\usepackage{tikz}
\tikzset{>=latex}
\usetikzlibrary{fit,automata,positioning,angles,quotes}

\makeatother

\usepackage{babel}
\providecommand{\corollaryname}{Corollary}
\providecommand{\definitionname}{Definition}
\providecommand{\examplename}{Example}
\providecommand{\lemmaname}{Lemma}
\providecommand{\problemname}{Problem}
\providecommand{\remarkname}{Remark}
\providecommand{\theoremname}{Theorem}

\usepackage{amsmath}               
  {
      \theoremstyle{plain}
      \newtheorem{assumption}{Assumption}
      \theoremstyle{plain}
      
  }
  
\SetSymbolFont{stmry}{bold}{U}{stmry}{m}{n}
\def\BibTeX{{\rm B\kern-.05em{\sc i\kern-.025em b}\kern-.08em
		T\kern-.1667em\lower.7ex\hbox{E}\kern-.125emX}}
		
\renewcommand{\baselinestretch}{0.95}

\begin{document}

\title{LQR-CBF-RRT*: Safe and Optimal Motion Planning

\thanks{$^{1}$Guang Yang and Amanda Prorok are with Computer Science, University of Cambridge, Cambridge, CB3 0FD, UK.
        {\tt\footnotesize gy268@cam.ac.uk }}

\thanks{$^{2}$Mingyu Cai is with Mechanical Engineering, Lehigh University, Bethlehem, PA, 18015 USA.
        {\tt\footnotesize mic221@lehigh.edu}}

\thanks{$^{3}$Ahmad Ahmad, Roberto Tron and Calin Belta are with Boston University
        {\tt\footnotesize \{ahmadgh,tron,cbelta\}@bu.edu }}

\thanks{The research is funded in part by a gift from Arm \textregistered. Their support is gratefully acknowledged.}

}

\author{ Guang Yang$^{1}$, Mingyu Cai$^{2}$, Ahmad Ahmad$^3$, Amanda Prorok$^{1}$, Roberto Tron$^3$ and Calin Belta$^3$
}
%"Learning for control with least-violating guarantees of infeasible LTL Specifications"
%"Learning minimum-violation guarantees of infeasible LTL Specifications" - This one might imply something else?
%"Model-free learning for minimally-violating control of infeasible LTL Specifications"
%"Learning control for infeasible LTL Specifications with minimum-violation guarantees"

% Deep Reinforcement Learning with Least-Violation Guarantees of Infeasible LTL Specifications for Model-free Continuous Navigation Control
\maketitle

\begin{abstract}
We present LQR-CBF-RRT*, an incremental sampling-based algorithm for offline motion planning. Our framework leverages the strength of Control Barrier Functions (CBFs) and Linear Quadratic Regulators (LQR) to generate safety-critical and optimal trajectories for 
a robot with dynamics described by an affine control system. CBFs are used for safety guarantees, while LQRs are employed for optimal control synthesis during edge extensions. Popular CBF-based formulations for safety critical control require solving Quadratic Programs (QPs), which can be computationally expensive. Moreover, LQR-based controllers require repetitive applications of first-order Taylor approximations for nonlinear systems, which can also create an additional computational burden. To improve the motion planning efficiency, we verify the satisfaction of the CBF constraints directly in edge extension to avoid the burden of solving the QPs. We store computed optimal LQR gain matrices in a hash table to avoid re-computation during the local linearization of the rewiring procedure. Lastly, we utilize the Cross-Entropy Method for importance sampling to improve sampling efficiency. Our results show that the proposed planner surpasses its counterparts in computational efficiency and performs well in an experimental setup. 
\end{abstract}

% \begin{IEEEkeywords}
% Formal Methods in Robotics and Automation, Deep Reinforcement Learning, Sampling-based Method
% \end{IEEEkeywords}

% \begin{keywords}
% Formal Methods, Robotic Planning, Multi-agent Systems, Optimal Control, Heterogeneous Teams
% \end{keywords}

\section{Introduction}
Robot motion planning involves computing an optimal plan that guides a robot safely and efficiently towards a goal. Sampling-based planners, such as Probabilistic Road Map (PRM) \cite{kavraki1996probabilistic} and Rapidly-exploring Random Trees (RRT) \cite{lavalle1998rapidly}, have been widely used to solve this problem. With particular relevance to this work, there is also a
an asymptotically optimal version of RRT, called RRT*, and its variants \cite{islam2012rrt,webb2013kinodynamic,li2020pq}. Given the popularity of RRT*-based planning algorithms, there has been an extensive effort to reduce the sampling computational complexity by exploiting the problem structure. An example is Informed RRT*, \cite{gammell2014informed}, which outperforms RRT* in terms of convergence rate. 

The above approaches assume a collision checking method exists during trajectory extension and sampling, such that unsafe states can be rejected during sampling to ensure safety. In addition, classic RRT or RRT* variants do not consider dynamics during planning, and therefore do not address the feasibility and control constraints. Several studies, such as those in\cite{bialkowski2013efficient} and \cite{pan2016fast}, have enhanced the collision checking procedure in sampling-based motion planners. However, these improvements can still be computationally expensive and require a more generalized solution for various nonlinear systems. 
Recent work~\cite{cai2023overcoming} employs geometric RRT* as guidance of deep reinforcement learning to improve the exploration efficiency. Similarly, the works in \cite{kobilarov2012cross,ahmad2022adaptive} employ an importance sampling (IS) algorithm, namely the cross-entropy method (CEM), for efficient exploration and sampling for RRT*. 
Our work aims to create a complete sampling-based motion planning framework that can efficiently generate samples from an optimal distribution, while ensuring optimality and safety. 

\noindent \textbf{Control Barrier Function (CBF)} There has been extensive research on safety-critical control and motion planning using CBFs. The significant advantage of CBFs comes from their guarantees on \textit{forward invariance} \cite{xu2015robustness}, i.e., if the system trajectory is initialized in a safe set, it will never leave the set if there exists a corresponding CBF. The standard formulation of a CBF-based controller involves Quadratic Programs (QPs) and applies the generated control inputs using zero-order hold (ZOH) controllers \cite{ames2016control}, whereby each QP is constrained by a Control Lyapunov Function (CLF) for stability and CBFs for safety (see \cite{ames2014control,wang2017safe,cheng2020safe}). For motion planning, the first CBF-based sampling-based motion planning algorithm was introduced in \cite{yang2019sampling}. In contrast to \cite{weiss2017motion,shkolnik2009reachability}, the \textit{forward invariance} property from CBFs removes the requirement of explicit collision checking. The approach can generalize to any control affine nonlinear system, and the safety sets require no assumption on linearity or convexity. Later works, such as \cite{manjunath2021safe, ahmad2022adaptive}, were built on top of it to improve its computational and sampling efficiency. Inspired by RRT* \cite{karaman2011sampling} and \cite{yang2019sampling}, the work in \cite{ahmad2022adaptive}, Adaptive CBF-RRT*, implemented a rewiring procedure to improve the cost of the trajectory. However, the computation is costly due to the iterative calculation of the QP.

\noindent \textbf{Linear Quadratic Regulator(LQR)} An LQR \cite{sun2016stochastic}, \cite{tedrake2010lqr} is  a widely used optimal control strategy that minimizes a quadratic cost function over a system's states and control inputs. The algorithm computes an optimal state-feedback gain based on the system dynamics, which can be solved by an algebraic Riccati equation. It has proven effective for robotic systems with linear dynamics \cite{argentim2013pid, okyere2019lqr, xiao2022learning,chen2019autonomous}. Other variations, such as iterative LQR (iLQR) \cite{van2014iterated, sun2016stochastic} can handle nonlinear system dynamics \cite{perez2012lqr,doerr2020motion,glassman2010quadratic}. In these methods, the cost function is approximated using a second-order Taylor expansion. While effective for nonlinear systems, this method can be costly. Moreover, collision checking is required during edge extension and rewiring.

In prior works \cite{yang2019sampling, ahmad2022adaptive}, CLFs and CBFs are incorporated into the steering function to ensure stability and safety. The QP is solved iteratively to generate controls that steer the system trajectory toward the sampled new state. The CBF constraints guarantee that the resulting trajectory avoids collisions with obstacles. However, this approach has a few downsides: (\textit{i}) Solving a sequence of QPs in the steering function can be costly, especially when the step size is selected to be small \cite{yang2019sampling}; (\textit{ii}) The QPs can easily become infeasible \cite{zeng2021safety}; and  (\textit{iii}) The QP controllers only ensure sub-optimal controls as the formulation of the objective function only guarantees optimality point-wise in time \cite{ahmad2022adaptive}.

\noindent \textbf{Contribution} We propose an efficient, optimal, and safe sampling-based offline motion planner that accounts for system dynamics. The generated state trajectory can then be tracked by feedback controllers online. This work specifically focuses on offline optimal planning using RRT*-like approaches \cite{perez2012lqr,islam2012rrt,li2020pq}. Therefore, dynamic obstacles and online planning are not within the scope of this paper. Our proposed approach is shown to significantly improve the efficiency from several aspects. In summary,
\begin{compactitem}
\item Our method generates optimal control during offline planning that minimizes the LQR cost \eqref{eq:LQRcost}. We show superior efficiency via baseline comparisons. For nonlinear systems, we reduce the frequency of computing LQR gains for locally linearized models by storing the previously calculated feedback gains using a hash table, avoiding repetitive LQR calculations during the rewiring procedure.

\vspace{0.8mm}

\item To reduce the computational cost and mitigate the infeasibility in the traditional CBF and CLF-based QP formulation in ~\cite{yang2019sampling, ahmad2022adaptive}, our framework does not require formulating and solving a CBF-based QP. Moreover, our approach guarantees optimality thanks to the LQR control formulation. 
\vspace{0.8mm}

\item  We used a customized omnidirectional robot to track the generated optimal trajectory for validating our method, and we showed that the robot successfully completed its navigation task in a cluttered environment.
\end{compactitem}

\section{Preliminaries}
\textbf{Notation:} A function $f: \mathbb{R}^n \mapsto \mathbb{R}^m$ is called \emph{Lipschitz continuous} on $\mathbb{R}^n$ if $ \exists \mathcal{L}\in \mathbb{R}^{+}$, such that $\|f(y)-f(x)\| \leq \mathcal{L} \|y-x\|, \forall x,y \in \mathbb{R}^n$. For a continuously differentiable function $h:\mathbb{R}^n \mapsto \mathbb{R}$, we use $\dot{h}$ to denote the derivative with respect to time $t$. The $\pounds^{r}_{f} h(x):=\frac{\partial \pounds^{r} h (x)}{\partial x} f(x), \pounds^{r}_{g} \pounds^{r-1}_{f}h(x):=\frac{\partial \pounds_f^{r-1} h(x)}{\partial x} g(x)$ are the $r$-th-order Lie derivatives \cite{khalil2002nonlinear}. We define continuous function $\alpha:[-b,a) \mapsto [-\infty,\infty)$, for some $a>0, b>0$, as a class kappa function, denoted as $\mathcal{K}$.  The $\alpha$ is strictly increasing, and $\alpha(0)=0$. Lastly, we denote $B_r (x)$ as a ball of region with radius r centered at $x \in \mathbb{R}^n$.

\subsection{System Dynamics}
Consider a continuous time dynamical system
\begin{equation}\label{eq:dynamicSystem}
\dot{x} = f(x) + g(x)u,
\end{equation} with state space $\mathcal{X}$ and control space $\mathcal{U}$. We define $x \in \mathcal{X} \subset \mathbb{R}^n$, $u \in \mathcal{U} \subset \mathbb{R}^m$, and $f(x):\mathbb{R}^n \mapsto \mathbb{R}^n$, $g(x):\mathbb{R}^n \mapsto \mathbb{R}^{n \times m}$ are locally Lipschitz continuous. We define the obstacles as the union of regions in $\mathcal{X}$ in which the robot configurations state variables coincide with the obstacles, $\mathcal{X}_{\mathrm{i, obs}} \subset \mathcal{X}, i=1,\dots,\mathrm{N_{obs}}$ and the obstacle-free set $\mathcal{X}_{\mathrm{safe}} := \mathcal{X} \setminus \cup_{i=1}^{\mathrm{N_{obs}}}\mathcal{X}_{\mathrm{i, obs}}$.
\vspace{-3.0 pt}
\subsection{Linear Quadratic Regulator}
We use LQR to compute optimal control policies. The cost function with the infinite time horizon is defined as
\begin{equation}\label{eq:LQRcost}
J = \int_{0}^\infty x^T Q x + u^T R u \:  dt,
\end{equation}
where $Q = Q^T \succeq 0$ and $R = R^T \succ 0$ are weight matrices for state $x$ and control $u$, respectively. For linear system dynamics, we can compute the closed form solution for the optimal control for a given linear-time invariant systems. Given a linear system, 
\begin{equation}\label{eq:linearSystem}
    \dot{x} = Ax + Bu,
\end{equation}
and the cost function \eqref{eq:LQRcost}, we can compute the LQR gain matrix $K_{LQR}$ by solving the algebraic Riccati equation with cost matrix $P$ as
\begin{equation}\label{eq:Reccati}
    A^T P + PA -PBR^{-1}B^TP+Q=0,
\end{equation}
such that optimal control policy is denoted as
\begin{equation}\label{eq:LQRPolicy}
    \pi^*(x) = - K_{\mathrm{LQR}}x, 
\end{equation}
where $K_{\mathrm{LQR}} = R^{-1}B^TP$ \cite{bertsekas2012dynamic} and $P$ is a stabilizing solution for \ref{eq:Reccati}. For nonlinear systems, a linearization can be made w.r.t. an equilibrium point $(x_{eq}, u_{eq})$ using first-order taylor expansion. Given a nonlinear system,
\begin{align*}
    \dot{x} &\approx F(x_{\mathrm{eq}},u_{\mathrm{eq}}) + \frac{\partial F(x_{\mathrm{eq}},u_{\mathrm{eq}})}{\partial x} (x-x_{\mathrm{eq}})\\ &+ \frac{\partial F(x_{\mathrm{eq}},u_{\mathrm{eq}})}{\partial u} (u-u_{\mathrm{eq}})
\end{align*}
Based on the equilibrium point , the linearized system can be re-written as $\dot{\hat{x}} = \hat{A} \hat{x} + \hat{B} \hat{u},$ with $\hat{A} = \frac{\partial F(x_{\mathrm{eq}},u_{\mathrm{eq}})}{\partial x}$ and $\hat{B} = \frac{\partial F(x_{\mathrm{eq}},u_{\mathrm{eq}})}{\partial u}$. Finally, we can compute the optimal gain $K_{\mathrm{LQR}}$ with the Riccati equation \eqref{eq:Reccati}.

\subsection{Higher Order Control Barrier Functions}
Given a continuously differentiable function  $h(x):\mathbb{R}^n \mapsto \mathbb{R}$. We denote the $r_b$-th derivative of $h(x)$ with respect to time $t$ of \eqref{eq:dynamicSystem} as 
\begin{equation} \label{eq:relative_degree}
h^{r_b}(x) = \pounds^{r_b}_{f} h(x)+\pounds^{r_b}_{g}\pounds^{r_b-1}_{f}h(x)u,
\end{equation}
The relative degree $r_\mathrm{b} \geq 0$ is defined as the smallest natural number such that $\pounds_{g} \pounds_{f}^{r_\mathrm{b}-1} h(x) u \neq 0$. Now, we formally introduce the definition of CBF. Given a time-varying function $h:\mathbb{R}^n \times [t_0, \infty) \mapsto \mathbb{R}$ that is $r_b^{th}$ order differentiable as
\begin{align}\label{eq:rb-order-functions}
   \Psi_{r}(x,t) = \dot{\Psi}_{r-1} +  \alpha_{r}(\Psi_{r-1}(x,t)), r=0,\dots,r_b.\\ \nonumber
\end{align}
From this, we denote a series of safety sets based on $\Psi_i$ as 
\begin{align}\label{eq:assosicateSets}
    \mathfrak{C}_{r} = \{x \in{\mathbb{R}^n}|\Psi_{r_b}(x,t)  \geq 0 \}, r=1,\dots,r_b. \\ \nonumber
\end{align}
\begin{definition} \cite{Xiao2019}
Given the functions defined in \eqref{eq:rb-order-functions} and safety sets \eqref{eq:assosicateSets}, the $r_b^{th}$ order function $h:\mathbb{R}^n \times [t_0, \infty) \mapsto \mathbb{R}$ is a Higher Order Control Barrier Function (HOCBF) for system \eqref{eq:dynamicSystem} if there exists class $\mathcal{K}$ functions $\alpha_1,\dots,\alpha_{r_b}$ such that 
\begin{equation}\label{eq:hocbf}
    \Psi_{r_b}(x(t),t) \geq 0,
\end{equation}
$\forall$ $(x,t) \in \mathfrak{C} \times [t_0,\infty) $, where $\mathfrak{C}:=\mathfrak{C}_0 \cap \dots \cap \mathfrak{C}_{r_b}$. Then, the state trajectory is always safe \cite{xu2015robustness}., i.e., the system trajectory is always safe.
\end{definition}
\begin{definition}\label{def:ForwardInvariance}
Given an initial state $x_0 = x(0)$, the set $\mathfrak{C}$ is called forward invariant if for every $x_0 \in \mathfrak{C}$, $x(t)\in  \mathfrak{C}, \forall t$. 
\end{definition}
\begin{thm} \cite{xu2015robustness}
Given system \ref{eq:dynamicSystem} and a continuous function $h$, if there exists a HOCBF 
as in \ref{eq:rb-order-functions}, then the set $\mathfrak{C}$ is forward invariant (i.e., the system is safe).
\end{thm}

\begin{definition}[Safe Trajectory]
    We define a safe trajectory for system (\ref{eq:dynamicSystem}) in $\mathfrak{C}$ as $\sigma :=(\mathbf{u},x)$, where given control inputs $\mathbf{u}:[0,T]\to\mathcal{U}$ the produced system trajectory $x(t)$ is subject to (\ref{eq:dynamicSystem}) and $x(t)\in\mathfrak{C}\subseteq \mathcal{X}_{\mathrm{safe}},\; \forall t\in[0,T]$. 
\end{definition}
\begin{definition}[Trajectory Cost]
    Given a produced control and state trajectory $\sigma = (\mathbf{u},x)$,  with $\mathbf{u}:[0,T]\to\mathcal{U}$, we define the state and control trajectory cost as follows. 
    \begin{equation} \label{eq:pathcost}
        \mathrm{c}(\sigma):= \int_{0}^T x^T Q x + u^T R u \:  dt,
    \end{equation}
where Q and R are the same weight matrices defined in \eqref{eq:LQRcost}.
\end{definition}

\section{Problem Formulation}
\begin{problem}\label{Problem1}
Given an initial state $x_{\mathrm{init}} \in \mathcal{X}_{\mathrm{init}} \subset \mathfrak{C}\subseteq \mathcal{X_{\mathrm{safe}}}$, where $\mathcal{X}_{\mathrm{init}}$ is an initial obstacle free set, and a bounded goal region  with $\mathcal{X}_{\mathrm{goal}} = B_{r_{\mathrm{goal}}} (x_{\mathrm{goal}})$ for some pre-defined radius $r_{\mathrm{goal}}$, such that $\mathcal{X}_{\mathrm{goal}} \subset \mathfrak{C}$. Find control inputs $\textbf{u}:[0,T] \mapsto \mathcal{U}$, where $T \in \mathbb{R}^{+}$ is the time horizon, that produces the path $\sigma$ such that $x(T) \in \mathcal{X}_{\mathrm{goal}}$ and $x(t) \in \mathfrak{C}, \forall t \in [0, T]$, with the optimal cost of the path, $\mathrm{C}^\ast(\sigma)$, being the optimal one in $\mathfrak{c}$, i.e.,
\begin{equation}
    \mathrm{c}^\ast(\sigma):= \mathrm{argmin}_{\mathbf{u}\in\mathcal{U}, \forall t \in [0,T],\; T \in \mathbb{R}^{+},\;(\ref{eq:dynamicSystem})} \mathrm{c}((\mathbf{u},x))
\end{equation}
\end{problem}

\section{LQR-CBF-RRT*}
In this section, we introduce \textbf{LQR-CBF-RRT*} algorithm \ref{alg:Ada_LQR-CBF-RRT*} as an extension of CBF-RRT \cite{yang2019sampling} that incorporates HOCBF constraints \eqref{eq:hocbf} checking for collision avoidance, and LQR policy \eqref{eq:LQRPolicy} for optimal control synthesis. We aim to solve \ref{Problem1} using sampling-based motion planning. Our framework encodes goal-reaching and safety requirements during the control synthesis procedure in the $\textit{steering}$ function (Line \ref{alg: steer1}, \ref{alg: steer2} in Algorithm \ref{alg:Ada_LQR-CBF-RRT*}). To ensure asymptotic optimality \cite{solovey2020revisiting}, an $\textit{rewiring}$ procedure (Line \ref{line:rewring_begin} in Algorithm \ref{alg:Ada_LQR-CBF-RRT*}) is included. For each iteration, a state, denoted as $x_{\mathrm{samp}}$ is sampled from an uniform distribution on the configuration space. The $x_{\mathrm{samp}}$ then used as an input in function $\texttt{NearbyNode}$ to find a set $\mathcal{X}_{\mathrm{near}}$ of its nearby nodes. The $\texttt{Nearest}$ is then used to select its closest neighbor node $x_{\mathrm{nearest}}$ from the tree. Next, the function $\texttt{LQR-CBF-Steer}$ \ref{algorithm:LQR-CBF-Steer} is used to synthesize controls based on computed optimal gain and steer the state from $x_{\mathrm{nearest}}$ to $x_{\mathrm{samp}}$. The resultant trajectory is denoted as $\sigma$ and its end node in $\sigma$ is $x_{\mathrm{new}}$. At the stage of $\texttt{ChooseParent}$, the function takes in the computed $x_{\mathrm{new}}$ and set  $\mathcal{X}_{\mathrm{near}}$ to compute state trajectories and corresponding costs from $x_{\mathrm{new}}$ to all the nearby nodes defined in $\mathcal{X}_{\mathrm{near}}$. The trajectory with the minimum cost, denoted as $\sigma_{\mathrm{min}}$, is used and its end node is defined as $x_{\textrm{min}}$. After this step, $x_{\mathrm{new}}$ and its edge are added to the tree. To asymptotically optimize the existing tree, the $\texttt{Rewire}$ function is used. During this procedure, the $\texttt{LQR-CBF-Steer}$ is used to check whether the cost for each nearby node $x_{\mathrm{nn}} \in \mathcal{X}_{\mathrm{near}}$ can be optimized. The $x_{\mathrm{nn}}$ that returns the minimum costs is rewired to $x_{\mathrm{new}}$, and the corresponding cost is updated. 

\subsection{LQR-CBF-RRT* Algorithm}
The detailed algorithm is introduced as the following. We define a tree $\mathcal{T} = (\mathcal{V},\mathcal{E})$, where $\mathcal{V}$ is a set of nodes and $\mathcal{E}$ is a set of edges. 
\begin{itemize}
\item \textbf{Nearby Node} utilizes a pre-defined euclidean distance $d$ to find a set of nearby nodes $\mathcal{X}_{\mathrm{near}}$ in $\mathcal{T}$ that is closest to $x_{\mathrm{sample}}$.

\item \textbf{Nearest} returns the nearest node from $\mathcal{X}_{\mathrm{near}}$ w.r.t. $x_{\mathrm{sample}}$.

\item \textbf{ChooseParent}
The procedure (Algorithm \ref{alg:Ada_LQR-CBF-RRT*} Line 10-14) tries to find collision-free paths between $x_{\mathrm{new}}$ w.r.t. all its neighboring nodes. If there exists a collision-free path between $x_{\mathrm{new}}$ and $x_{\mathrm{nn}} \in \mathcal{X}_{\mathrm{near}}$, the corresponding cost is calculated. The \texttt{ChooseParent} procedure then selects the $x_{\mathrm{nn}}$ with the lowest cost \eqref{eq:LQRcost} as the parent of $x_{\mathrm{new}}$.
  
 \item \textbf{Rewire} is defined in Algorithm \ref{alg:Ada_LQR-CBF-RRT*} Line 16-20. It evaluates and optimizes the LQR cost \eqref{eq:LQRcost}. This function checks a selected node's neighborhood and calculates the costs w.r.t. all the neighboring nodes. A new state trajectory from the current node to the nearby node is added to $\mathcal{T}$.

\item \textbf{LQRSolver}
The method (Algorithm \ref{alg:Ada_LQR-CBF-RRT*} Line 4) computes optimal gain matrix $K_{\mathrm{LQR}}$ from Riccati equation \eqref{eq:Reccati}. For a nonlinear system, the system can be linearized locally \cite{perez2012lqr} and then computes $K_{\mathrm{LQR}}$ based on linearized system dynamics.
  \item \textbf{LQR-CBF-Steer} The steering function (Algorithm \ref{algorithm:LQR-CBF-Steer}) contains two components: (i) The LQR controller (ii) CBF safety constraints. Given two states $(x_{\mathrm{current}},x_{\mathrm{next}})$, the LQR controller generates a sequence of optimal controls $u^*_i$, that steers state trajectory based on \eqref{eq:linearSystem}. At each time steps, the CBF constraints \eqref{eq:hocbf} are checked to ensure generated path is collision free.  If none of the CBF constraints are active, then the \texttt{LQR-CBF-steer} is used for steering $x_{\mathrm{current}}$ to $x_{\mathrm{next}}$, and node $x_{\mathrm{next}}$ is added to the tree. Otherwise, the extension stops at the first encounter of \eqref{eq:hocbf} that is violated. Then, the end state $x_{\mathrm{new}}$ and its trajectory $\sigma$ are added to the tree.

\end{itemize}
\begin{rem}
Thanks to the \textit{forward invariance} property from CBFs \cite{yang2019sampling}, no explicit collision checking is required. The generated state trajectory and controls from \textbf{LQR-CBF-Steer} are guaranteed to be safe. 
\end{rem}

\setlength{\textfloatsep}{15pt}
\begin{algorithm}[t]\scriptsize
	\textbf{Initialization:}   $\mathcal{T} = (\mathcal{V},\mathcal{E})$, $\mathcal{V} \leftarrow$ $\{x_{\mathrm{init}}\}$; $\mathcal{E} \leftarrow \emptyset$; $i = 0$; $\,\texttt{adapFlag}=\texttt{True}$, $\,\texttt{optDensityFlag}\,=\,\texttt{False}$, and $r = \eta$\label{line:alg1_init}\\
	\While{$i<N$}
	{$x_{\mathrm{samp}}\leftarrow\texttt{Sample}(\mathcal{G},\mathcal{V},\texttt{adapFlag})$ \label{Line: Alg1.Sampling}\\
            $x_{nearest} \leftarrow \texttt{Nearest}(\mathcal{V},x_{\mathrm{samp}})$\label{line:1st_v_nearest}\\
            $x_{\mathrm{new}}, \sigma \leftarrow \texttt{LQR-CBF-Steer}(x_{\mathrm{nearest}}, x_{\mathrm{samp}})$\label{alg: steer1}\\ 
            $\mathcal{X}_{near}\leftarrow\texttt{NearbyNode}(\mathcal{V},x_{nearest})$ \label{alg:nearbyNode}\\
		$r = \min\{\lambda(\log(|\mathcal{V}|)/|\mathcal{V}|)^{1/(d+1)},\eta\} $\label{line:alg1_ball_radi}\\
		$\mathcal{X}_{near}\leftarrow\texttt{Near}(\mathcal{V},r,x_{new})$\\
		%%% Choose Parent:
		minCost $\leftarrow$ $\infty$; $x_{\mathrm{min}},     \sigma_{\mathrm{min}} \leftarrow \texttt{None}, \texttt{None}$\label{line:best_parent_begin}\\
            \ForEach{$x_{\mathrm{near}}\in \mathcal{X}_{\mathrm{near}}$}
		{	$\sigma$ $\leftarrow$ \texttt{LQR-CBF-Steer}                          ($x_{\mathrm{new}}$, $x_{\mathrm{nn}}$)\label{alg: steer2}\\
			\If{$x_{\mathrm{nn}}.\texttt{cost}$+Cost($\sigma$) $<$ minCost}
			{
				minCost $\leftarrow$ $x_{\mathrm{nn}}.\texttt{cost}$+Cost($\sigma$)\\
				$x_{\mathrm{min}}$ $\leftarrow$ $x_{\mathrm{nn}}$; $\sigma_{\mathrm{min}} \leftarrow \sigma$
			}
		\label{line:best_parent_end}}
		$\mathcal{V}\leftarrow\mathcal{V}\cup\{x_{new}\}$; $\mathcal{E} \leftarrow \mathcal{E} \cup \{x_{\mathrm{min}}, x_{\mathrm{new}}\}$\\
		%%%%%%% Rewiring procedure:
		\ForEach{$x_{\mathrm{nn}} \in \mathcal{X}_{near}$} 
		{  \label{line:rewring_begin}
                $\sigma$ $\leftarrow$ \texttt{LQR-CBF-Steer}($x_{\mathrm{new}}$, $x_{\mathrm{nn}}$)\\
			\If{$x_{\mathrm{new}}.\texttt{cost}$+Cost($\sigma$) $<$ $x_{\mathrm{nn}}.\texttt{cost}$}
			{ $x_{\mathrm{nn}}.\texttt{parent} \leftarrow 
        x_{\mathrm{new}}$\\
				$x_{\mathrm{new}}.\texttt{cost} = $Cost($x_{\mathrm{new}}$)
			}
		\label{line:rewring_end}}
		$\mathcal{T},\;\mathcal{G}\leftarrow\texttt{extToGoal}(\mathcal{V},x_{\mathrm{new}},\texttt{adapFlag})$;	$i\leftarrow i+1$ \label{line:alg1_ CBF-RRT*_ext2goal}}\Return$\mathcal{T}$
	\caption{LQR-CBF-RRT$^\ast$}
	\label{alg:Ada_LQR-CBF-RRT*}
\end{algorithm}
\vspace{-10 pt}
\begin{algorithm}[t]\scriptsize
	$\mathbf{x} \leftarrow \texttt{None}, \mathbf{u} \leftarrow \texttt{None}$\\
        $\mathbf{x}.\texttt{add}(x_{\mathrm{current}})$\\
        $\mathrm{K_{LQR}} \leftarrow \texttt{LQRsolver}(x_{\mathrm{current}}, x_{\mathrm{next}})$\\
        \ForEach{$t < \mathrm{T}$}
		{$x^{'}, u  \leftarrow \texttt{Integrator}(x_{\mathrm{current}}, \mathrm{K_{LQR}})$\\
	       \eIf{$\textrm{Satisfied} \leftarrow \texttt{CBFconstraints}(x^{'}, u)$}
			{$\mathbf{x}.\texttt{add}(x^{'})$; $\mathbf{u}.\texttt{add}(u)$\\
				$x_{\mathrm{current}} \leftarrow x^{'}$
			}
                {
                \Return $\sigma = (\mathbf{x},\mathbf{u})$
                }
		}
	\caption{LQR-CBF-Steer($x_{\mathrm{current}}, x_{\mathrm{next}}$)}
	\label{algorithm:LQR-CBF-Steer}
\end{algorithm}
\setlength{\textfloatsep}{14pt}
\setlength{\textfloatsep}{15pt}
\begin{algorithm}[t]\scriptsize
	$\texttt{u}\sim \texttt{Uniform}(0,1)$\\
	\eIf{$\texttt{u}\;\leq 0.5\; \wedge\;\mathcal{G}\neq\emptyset$}
		{
			\eIf{\texttt{optDensityFlag}}
				{
					$\mathcal{X}\sim\hat{\mathfrak{g}}^\ast(x)$\\
					\Return $(x)$\label{line:alg_ smple_from_opt_est}
				}
				{
					\eIf{$\texttt{mod}(|\mathcal{G}|, n_v)=0$}
						{
					    $\mathfrak{E}\leftarrow\texttt{Quantile}(\mathcal{G},\varrho)$\Comment{Assign the elite set}\\\label{line:alg2_strt_discrtzng}
						$\hat{\mathfrak{g}}(x)\leftarrow\texttt{CE\_Estimation}(\mathfrak{E},m)$\Comment{Compute PDF of $\mathfrak{E}$}\label{line:Alg2_CE_estimation}
							\\ \Return$x\sim\hat{\mathfrak{g}}(x)$ \label{line:alg2_smpl_form_g_est}
						}%Else:
						{
							\Return $x\sim\texttt{Uniform}(\mathcal{X})$
						}
				}
		}{\Return $x\sim\texttt{Uniform}(\mathcal{X})$}
	\caption{$x_{\mathrm{samp}}\leftarrow\texttt{Sample}(\mathcal{G},\mathcal{T},m)$}
	\label{alg:Sample}
\end{algorithm}
\setlength{\textfloatsep}{14pt}

\section{Efficiency Improvements}
In this section, we develop three strategies to improve the efficiency of the sampling-based process. Notably, when integrating RRT* with both linear and nonlinear systems, section~\ref{subsec: LQR} shows how the designed LQR steering function has better performance compare to offline MPC and iLQR. Moreover, since repetitively solving CBF-QP is computationally expensive, section~\ref{subsec: CBF} designs a QP-free mechanism to check CBF constraints. Lastly, we propose an adaptive sampling method to enhance efficiency and optimality in \ref{sec: experiment}.

\subsection{Efficient LQR computation~\label{subsec: LQR}}
For each steering process, applying the MPC~\cite{camacho2013model} and iLQR~\cite{tassa2012synthesis} requires iteratively solving optimization problems multiple times to obtain an optimal control. These approaches are not efficient as sampling-based methods need to keep using steering function to explore the state space. 

This leads to the fundamental motivation for applying LQR in our framework, i.e., providing an efficient steering process.
For linear dynamic systems, $\dot{x} = Ax + Bu$. The LQR feedback gain $K_{\mathrm{LQR}}$ can be computed through Riccati equation in~\eqref{eq:Reccati} i.e., $K_{\mathrm{LQR}} = LQR(A, B, Q, R)$, which only depends on pre-modeled matrices $A, B, Q, R$. We only need to compute it once and use it for every steering process.

For nonlinear systems, we can linearize it around a local goal as the equilibrium point. To reduce the computational burden, we only solve for the gain once for each steering process. While it outperforms offline MPC and iLQR in efficiency, the computation time can be further improved. More specifically, the linearization and computation of optimal control gain in the approximated dynamical system during the \textbf{Rewiring} procedure is costly. 
To mitigate this issue, we store the LQR feedback gain for each local goal of the steering process using a hash table to reduce computation during the rewiring procedure (lines 16 - 18 in Algorithm \ref{alg:Ada_LQR-CBF-RRT*})
\subsection{LQR--CBF-RRT* without solving QP~\label{subsec: CBF}}
The existing works \cite{yang2019sampling, ahmad2022adaptive} formulate the QP controllers and iteratively solve for the controls point-wise in time during \texttt{ChooseParent} and \texttt{Rewire} procedure, which can leads to the following problems: First, it requires extensively solving the QPs over time. Second, hyper-parameters have to be determined to ensure the feasibility of the QP \cite{zeng2021safety}.

In this work, instead of solving the CBF-CLF-QP or CBF-QP with minimum perturbations of a given reference controllers, we only check whether the CBF constraints~\eqref{eq:relative_degree} are satisfied and use them as the termination condition of the steering process. For example, given a reference trajectory $x_{0}u_{0}x_{1}\ldots u_{n-1}x_{n}$ from LQR optimal control, to reach the sampled goal from current state $x_{current} = x_{0}$. 
We check if moving from $x_{i}$ to $x_{i+1}$ via control $u_{i}$ for all $i\in \{0, 1,\ldots, n-1\}$ satisfies the CBF condition. We terminates the extension process if the CBF condition is violated at any step $i$. Different from~\cite{manjunath2021safe}, we do not ignore the whole reference path if it violates the CBF condition. Instead, we keep the prefix trajectory where the CBF constraint is always satisfied.

% we continue to extend the edge until a pre-defined step is reached, as long as CBF constraints are satisfied. The extension process terminates when we first encounter the violation of the CBF constraints. The checking of the CBF violation is a constant time procedure.

\subsection{Adaptive Sampling}
\label{sec:Adaptive_smplng}
We leverage LQR-CBF-RRT$^\ast$ with an adaptive sampling procedure \cite{ahmad2022adaptive} to focus sampling  promising regions of $\mathcal{X_{\mathrm{safe}}}$ in order to approximate the solution of Problem \ref{Problem1} with a fewer number of samples. We define $\mathcal{G}$ as set of control and state trajectories pairs, i.e., $\mathcal{G} := \{(\mathbf{x(t)}, \mathbf{u(t)})| x(t) \in \mathcal{X}_{\mathrm{safe}}, \forall t \in [0,T]\}$. The CEM \cite{Rubinstein1999} is used for IS, which is a multi-stage stochastic optimization algorithm that iterates upon two steps: first, it generates samples from a current Sampling Density Function (SDF) and computes the cost of each sample; second, it chooses an \emph{elite subset}, $\mathfrak{E}$, of the generated samples for which their cost is below some threshold; finally, the elite subset is used to estimate a probability density function (PDF) as if they were drawn as i.i.d samples and we define $m$ to be the number of elite trajectories. The CEM was first used in \cite{Kobilarov2012e} for RRT$^\ast$ importance sampling with Gaussian mixture models (GMM). We use the CEM IS procedure that we implemented in \cite{ahmad2022adaptive} in this paper (Algorithm \ref{alg:Sample}), since it utilizes weighted Gaussian kernel density estimate (WGKDE) for estimating the PDF of the elite samples. The number of mixtures in GMM has to be picked based on the workspace and is difficult to tune. Using WGKDE, however, mitigates this challenge. 

\section{Probabilistic Completeness and Optimality}
We provide the probabilistic proof for our algorithm. Following \cite{karaman2011sampling} we assume that Problem \ref{Problem1} is \textit{robustly feasible} with minimum clearance $\varepsilon>0$. Hence, $\exists\mathbf{u}\in\mathcal{U}$ such that it produces a system (\ref{eq:dynamicSystem}) trajectory $\phi$, where $\phi: [0,T] \mapsto \mathcal{X}_{\mathrm{safe}}$, $\phi(0)=x_{\mathrm{init}}$, $\phi(T) = x_{\mathrm{goal}}$, and $\forall t \in [0,T];$ $\Psi_{r_b}(\phi(t)) \geq 0 , \forall \phi \in \mathfrak{C}_0  \cap \dots \cap \mathfrak{C}_{r_b}$.

\begin{rem}
The assumption on the sets of initial state and goal state are required to ensure the existence of a solution. Without such assumption, there's a probability such that certain states $x_{\mathrm{init}},x_{\mathrm{goal}}$ will be rejected, and the state trajectory will never reach the goal. In practice, we can ensure the assumption holds by choosing appropriate HOCBF hyper-parameters. 
\end{rem}
\begin{lem}\label{lemm1}
\cite{li2016asymptotically} Given two trajectories $\phi$ and $\phi'$, as well as a period $T \geq 0$, such that $\phi(0)=\phi'(0)=x_{\mathrm{init}}$, the trajectories can be bounded by control
\begin{equation}\label{eq:trajbound}
    \lVert \phi'(T) - \phi(T) \rVert < K_u \cdot T \cdot e^{K_x \cdot T}\cdot \sup(\lVert u(t) - u'(t)\rVert),
\end{equation}
for $K_u, K_x \in \mathbb{R}> 0$.
\end{lem}
Lemma \ref{lemm1} ensures any two trajectories with the same initial state, the distance between their end states at time $T$ is bounded by the largest difference in their controls. 
\begin{rem}\label{rmrk:TunableParamsOfCBFs}
Consider steering from any $x_s\in\mathcal{X}_{safe}$ towards any reachable $x_f\in\mathcal{X}_{safe}$ using Algorithm \ref{algorithm:LQR-CBF-Steer}. By carefully selecting appropriate constants of the class $\mathcal{K}$ functions, $\alpha_1,\,...,\,\alpha_{\rho}$, of the HOCBF (see Defintion \ref{eq:hocbf}), we can produce $\textbf{u} = u(t_0), t(t_1),\dots,u(T_{OL})$, such that the bound \eqref{eq:trajbound} be given as $||x_f - x_f^\prime||=\mu<\frac{\varepsilon}{4}$, where $x_f$ and $x^\prime_f$  are the states at time $T_{OL}$ of the produced trajectories under the control inputs $\textbf{u}_{OL}$ and the CBF-QP control inputs $\textbf{u}_{LP}$, respectively, and $\mu\in\mathbb{R}_{>0}$.
\end{rem}
\begin{thm}
The LQR-CBF-RRT* (Algorithm \ref{alg:Ada_LQR-CBF-RRT*}) is probabilistically complete. 
\end{thm}
\begin{proof}
The completeness of RRT* is implied by the completeness of RRT \cite{karaman2011sampling}. Therefore, we need to prove the completeness of LQR-CBF-RRT. One can implement LQR-CBF-RRT$^\ast$ by mitigating Lines \ref{line:best_parent_begin}-\ref{line:best_parent_end} and Lines \ref{line:rewring_begin}-\ref{line:alg1_ CBF-RRT*_ext2goal} in Algorithm \ref{alg:Ada_LQR-CBF-RRT*}. 
 Given that the local motion planner \ref{algorithm:LQR-CBF-Steer}, and by leveraging Theorem 2 in \cite{RRTcompleteness}, we need to prove that the incremental state trajectory will propagate to a sequence of intermediate states until reaching $\mathcal{X}_{\mathrm{goal}}$. Assuming the trajectory $\phi$ of the solution of Problem \ref{Problem1} with $\varepsilon$ clearance and has a length $L$. Considering $m+1$ equidistant states $x_{i}\in\phi,\;i=1,\dots,m+1$, where $m=\left\lfloor\frac{4L}{\varepsilon} \right\rfloor$, we define a sequence of balls with radius $\varepsilon/4$ that are centered at these states. For state $x_i$, such ball is given by: $\mathfrak{B}_{\frac{\varepsilon}{4}}(x_i):=\{x_b\;|\;||x_i-x_b||\leq\frac{\varepsilon}{4}\}$. For the consecutive states $x_{i},x_{i+1}\in\varphi_{x}$, we want to prove that starting from $x_s\in\mathfrak{B}_{\frac{\varepsilon}{2}}(x_i)$ the steering function $\texttt{LQR-CBF-Steer}$ is able to generate a control and state trajectories that its end state $x_f^\prime$ fall in $\mathfrak{B}_{\frac{\varepsilon}{4}}(x_{i+1})$. Given Remark \ref{rmrk:TunableParamsOfCBFs}, we assign $\eta = \frac{\varepsilon}{4}+\mu+2\iota$ and $0<\iota<\frac{\varepsilon}{4}-\mu$. Next, we assign $\mathfrak{B}_{\eta}(x_s)$ and $\mathfrak{B}_{\frac{\varepsilon}{4}-\mu-\iota}(x_{i+1})$ at $x_s$ and $x_{i+1}$, respectively. Let $\mathcal{S}:=\mathfrak{B}_{\eta}(x_s)\cap\mathfrak{B}_{\frac{\varepsilon}{4}-\mu-\iota}(x_{i+1})$ denotes the successful potential end-states set. For any $x_f\in\mathcal{S}$, $\texttt{LQR-CBF-Steer}$ generates trajectories that fall in $\mathfrak{B}_{\mu}(x_f)\subset\mathfrak{B}_{\frac{\varepsilon}{4}}(x_{i+1})$. We denote $|.|$ as the Lebesgue measure, then, for $x_s$, the probability of generating states in $\mathcal{S}$ is $p=\frac{|\mathcal{S}|}{|\mathcal{X}|}$ and is strictly positive. The probability $p$ can be represented as success probability of the $k$ Bernoulli trials process \cite{RRTcompleteness} that models generating $m$ successful outcomes of sampling states that incrementally reach $\mathcal{X}_{goal}$. The rest of the proof follows the proof of Theorem 1 in \cite{RRTcompleteness}.
\end{proof}
Next, we tackle the asymptotic optimality with respect to a safety region, which is conservative and we are able to compute.  
\begin{assumption}\label{assum:lqr-cbf-steer_in_C}
    We assume that we have access to a (conservative) safety set $\mathfrak{C}\subseteq \mathcal{X}_{safe}$ in which the procedure LQR-CBF-Steer$(x_{\mathrm{current}},x_{\mathrm{next}})$ is able to produce a trajectory that converges to $x_{\mathrm{next}}$.  
\end{assumption}

Assumption \ref{assum:lqr-cbf-steer_in_C} implies that the optimal path on the original state space is not realizable by the proposed algorithm due to the CBF constraints. Rather, the algorithm will produce an asymptotically optimal path in a conservative set of the environment. Provided that our algorithm is an RRT$^\ast$ variant with an LQR-based local motion planner with a CBF-based sample rejection method, in the following theorem, we provide an asymptotic optimality result. At the $i$-th iteration of LQR-CBF-RRT*, we consider the produced path, $\sigma^i$, which connects $x_{\mathrm{init}}(0)$ to an $x(T)\in\mathcal{X}_{\mathrm{goal}}$, as a concatenation of paths produced by $\texttt{LQR-CBF-Steer}$.  

\begin{thm}
Consider the cost of the optimal solution of Problem \ref{Problem1} to be $\mathrm{c}^\ast(\sigma)$; the LQR-CBF-RRT* (Algorithm \ref{alg:Ada_LQR-CBF-RRT*}) produces an \textit{asymptotically optimal} solution in $\mathfrak{C}$ to Problem \ref{Problem1}. i.e., 
\begin{equation}
    P(\mathrm{lim}_{i\to\infty}\mathrm{c}(\sigma^i)=\mathrm{c}^\ast(\sigma)) = 1
\end{equation}
\end{thm}
\begin{proof} (Sketch) Given Assumption \ref{assum:lqr-cbf-steer_in_C}, the asymptotic optimality of the solution follows directly from Theorem 5 in \cite{karaman2010KinoDynRRTstar}  
\end{proof}

\begin{figure}[!t]\centering
	\subfloat[]{{\label{Fig2:a}\includegraphics[width=0.50\columnwidth]{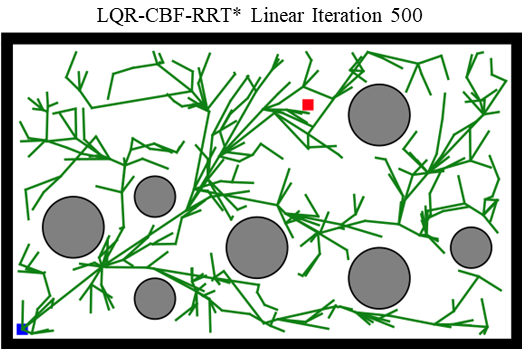} }}%
	\subfloat[]{{\label{Fig2:b}\includegraphics[width=0.50\columnwidth]{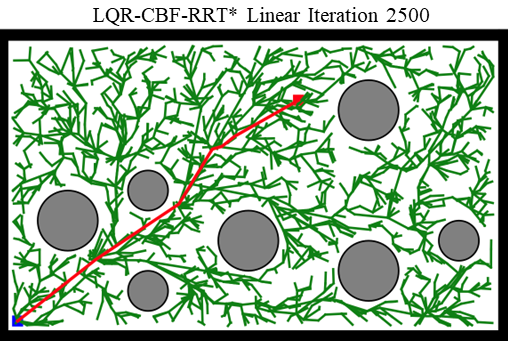} }}%
 $\qquad$
 	\subfloat[]{{\label{Fig2:c}\includegraphics[width=0.50\columnwidth]{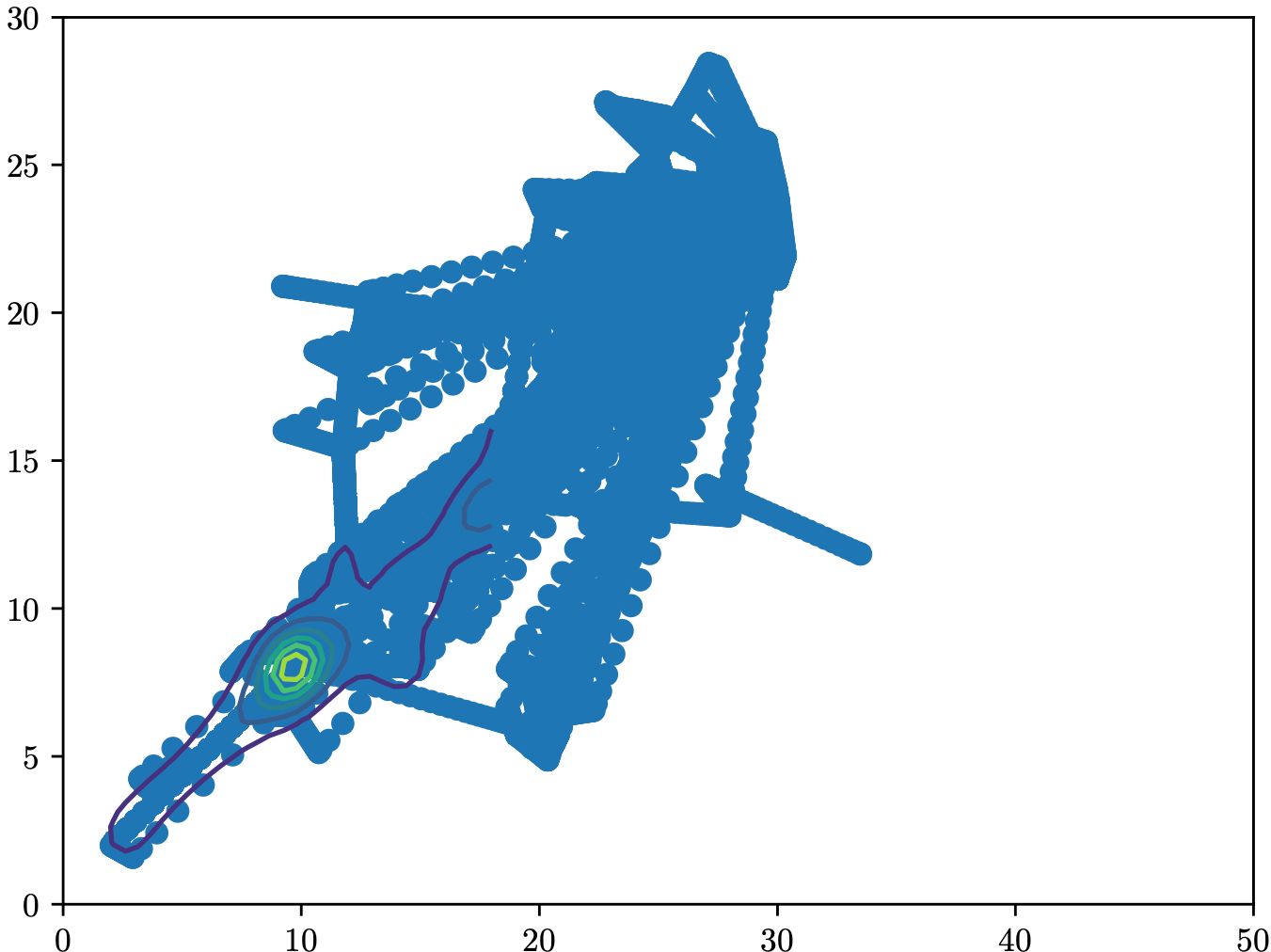} }}%
	\subfloat[]{{\label{Fig2:d}\includegraphics[width=0.50\columnwidth]{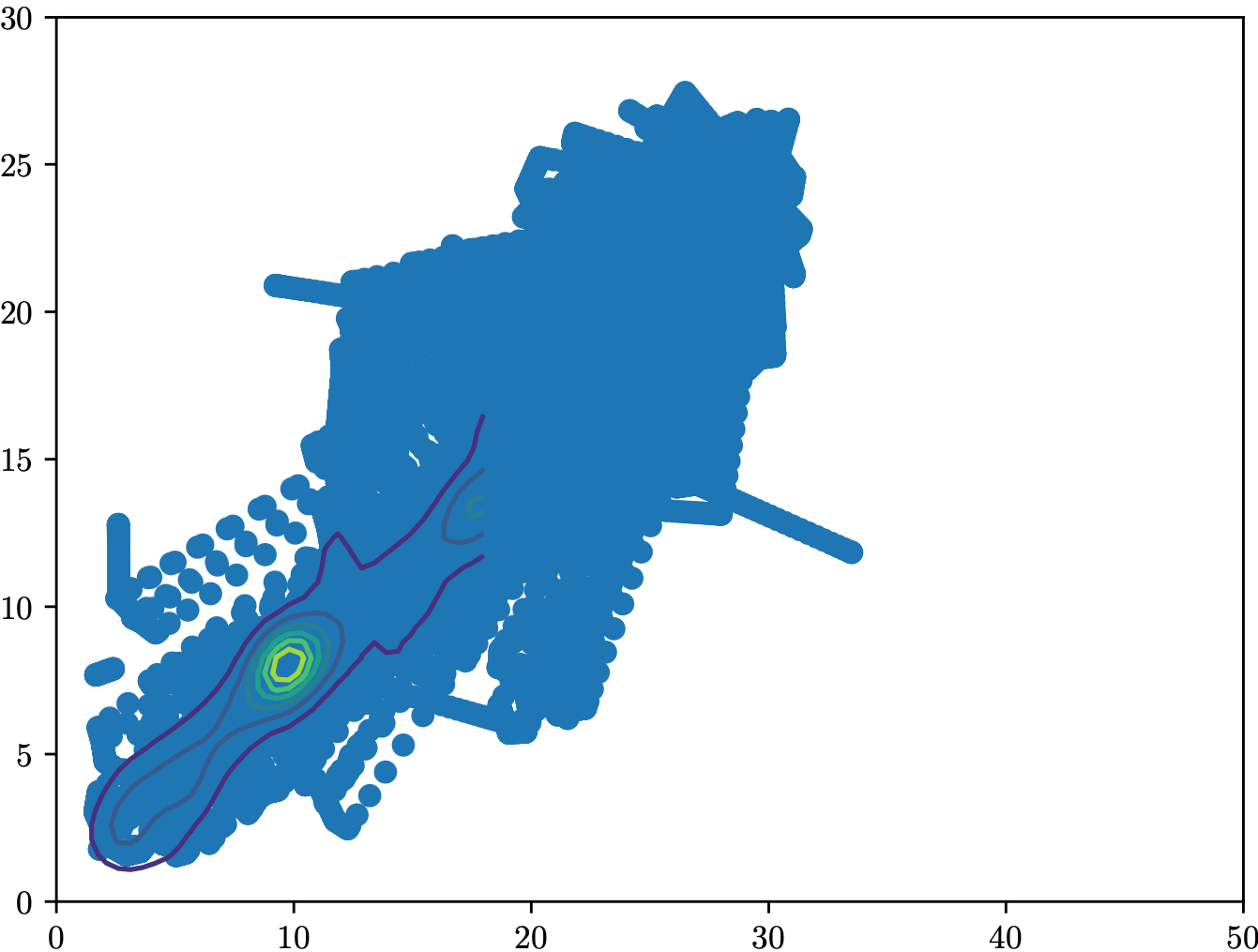} }}%
	\caption{We performed simulations on a double integrator (\ref{Fig2:a} to \ref{Fig2:d}  with adaptive sampling. The plots \ref{Fig2:a} (500 steps) and \ref{Fig2:b} (2500 steps) demonstrate how our algorithm explores the workspace. The \ref{Fig2:c} and \ref{Fig2:d} illustrate how SDF distribution is generated as the number of elite samples increases.}
	\label{fig:linear_example}
\end{figure}
\vspace{-\baselineskip}
\begin{figure}[t]\centering
	\subfloat[]{{\label{Fig2:e}\includegraphics[width=0.50\columnwidth]{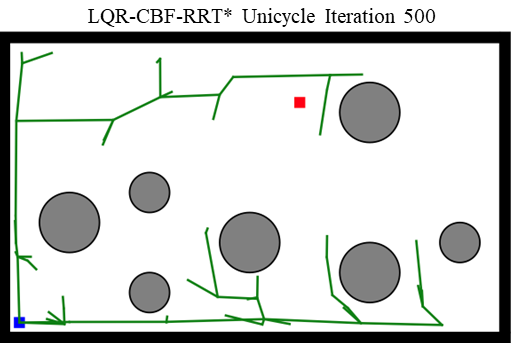} }}%
	\subfloat[]{{\label{Fig2:f}\includegraphics[width=0.50\columnwidth]{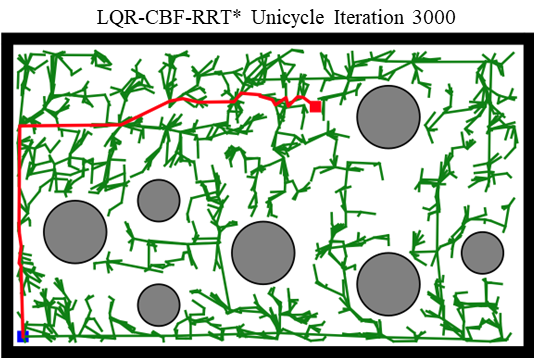} }}%
 $\qquad$
 	\subfloat[]{{\label{Fig2:g}\includegraphics[width=0.50\columnwidth]{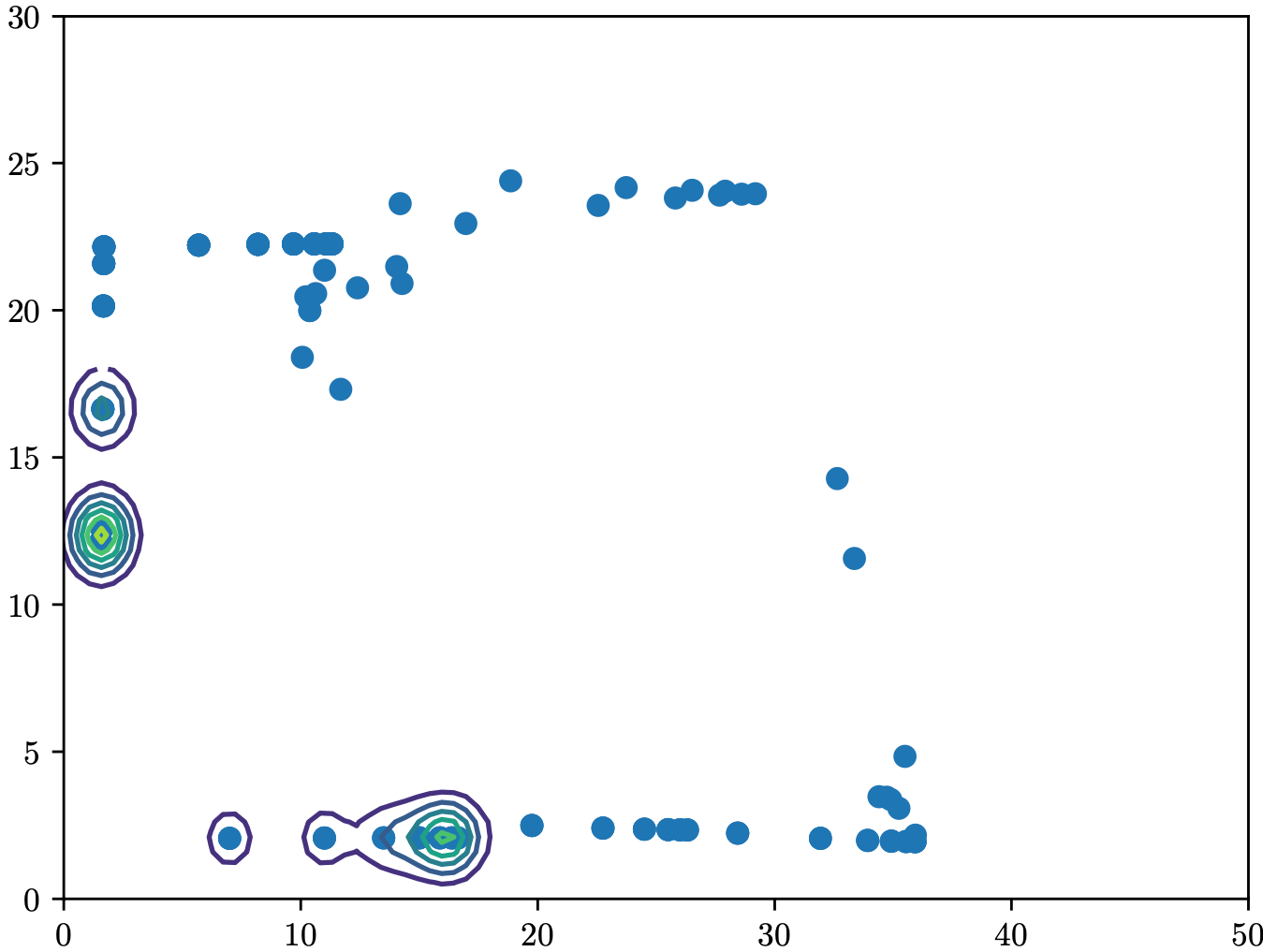} }}%
	\subfloat[]{{\label{Fig2:h}\includegraphics[width=0.50\columnwidth]{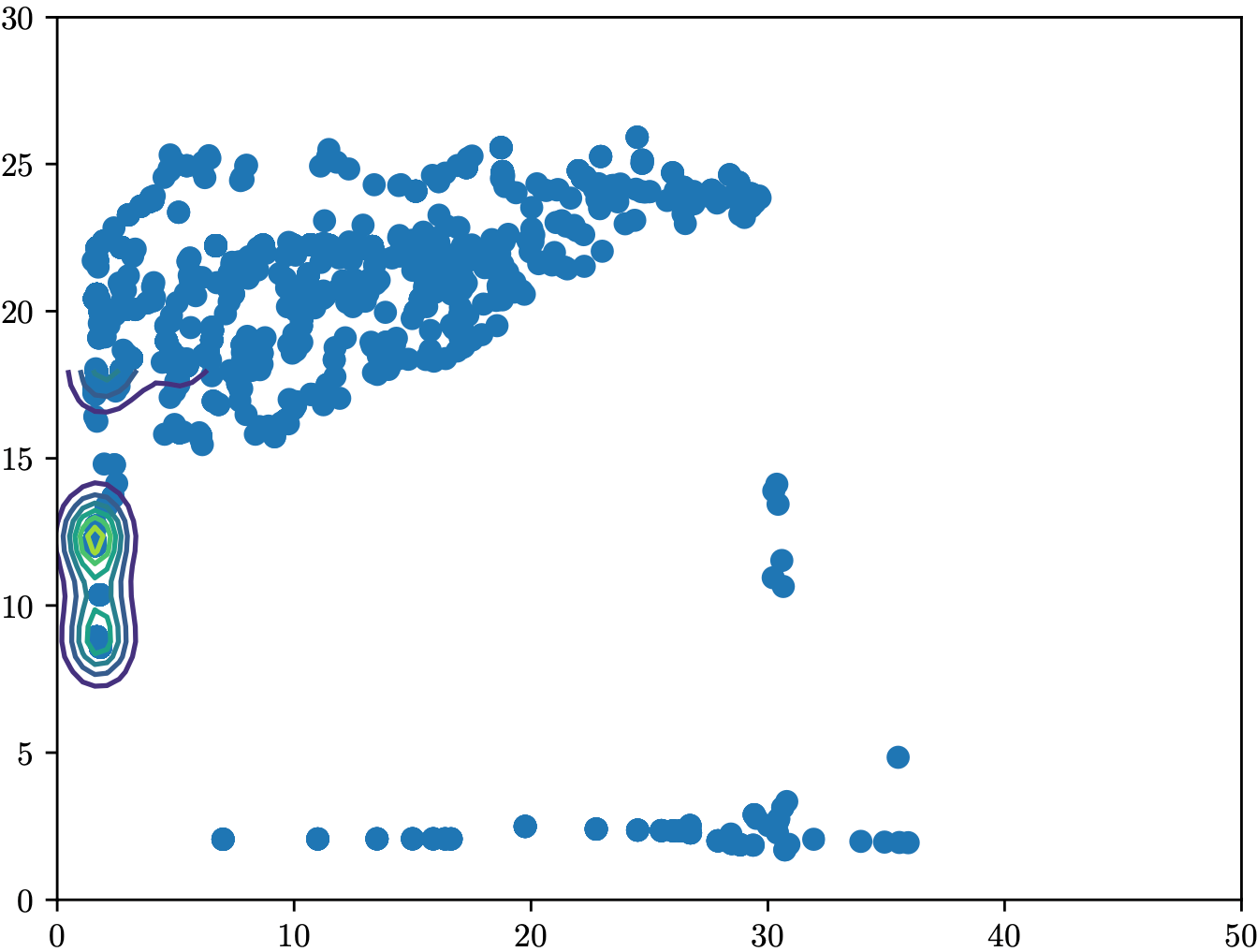} }}%
	\caption{ We performed the simulation on an unicycle model ( \ref{Fig2:e} to \ref{Fig2:h}) with adaptive sampling.
 The same environment configuration applies to the nonlinear system, where \ref{Fig2:e} (500 steps) and \ref{Fig2:f} (3000 steps) show the generated state trajectories. The \ref{Fig2:g} and \ref{Fig2:h} show the SDF level set and elite samples for the unicycle model.}
	\label{fig:unicycle_example}
\end{figure}

\section{Experimental Results~\label{sec: experiment}}
In this section, we illustrate our framework on two different systems. All experiments have the same workspace configuration, with initial state $x_{\mathrm{init}} = [2,2]$ and the goal state $x_{\mathrm{goal}}=[30,24]$. We consider obstacles to be circular and the simulations are performed on a MacBook Pro with an M1 Pro CPU. The code is available at ~\footnote{\url{https://github.com/mingyucai/LQR_CBF_rrtStar}}.

\subsection{Baseline Summary}
We performed several metrics comparison on our framework, including avoiding the construction QP, storing previous LQR feedback gain, and adaptive sampling. For the nonlinear system, we have (1) Our method.  (2) LQR-CBF-RRT* without storing feedback gain for nonlinear systems, i.e., naive adaptive LQR-CBF-RRT*; (3)  LQR-CBF-RRT* without both adaptive sampling and storing feedback gain. (4) QP based LQR-CBF-RRT*. 
For the linear system, we have (1) Our method. (5) LQR-CBF-RRT* with QP-solver during steering processes, i.e., LQR-CBF-RRT*-QP; The reason that we uses a separate system to do the comparison is because the QP controller with the nonlinear system is too conservative, i.e., many of the QPs could become infeasible. 

\noindent \textbf{Nonlinear System} We conduct a performance comparison on an unicycle model. We perform five experiments with 2000 iterations for each baseline with different random seeds. In table \ref{tb:example3}, we list the time it takes to complete the motion planning for different baselines.
\begin{figure}
    \centering
    \includegraphics[width=0.65\linewidth]{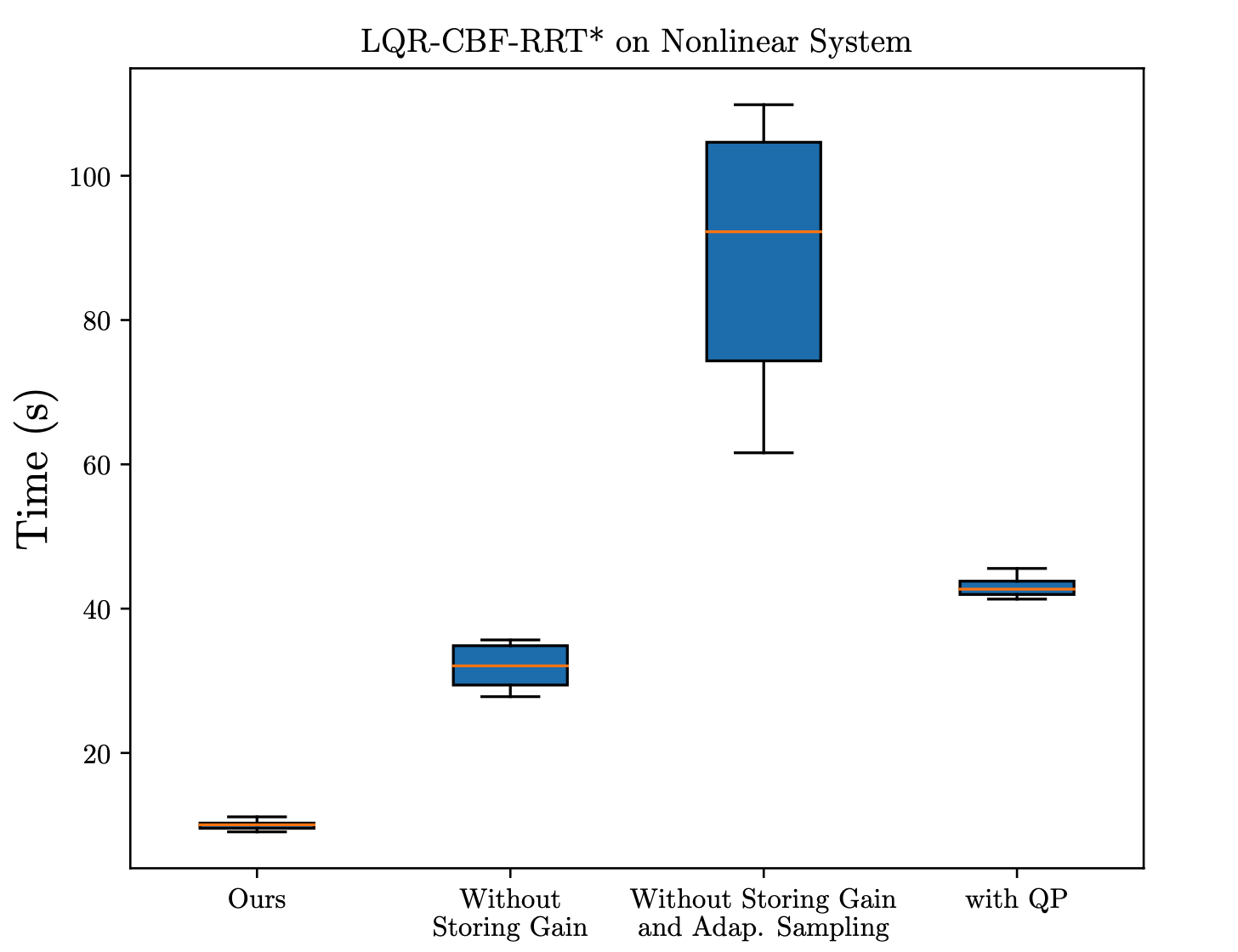}
    \caption{The figure shows our method outperforms the other three cases by a significant margin. The last column (with QP) was the method implemented in \cite{yang2019sampling}.}
    \label{fig:benchmark_time}
\end{figure}
\begin{figure}
    \centering
    \includegraphics[width=0.8\linewidth]{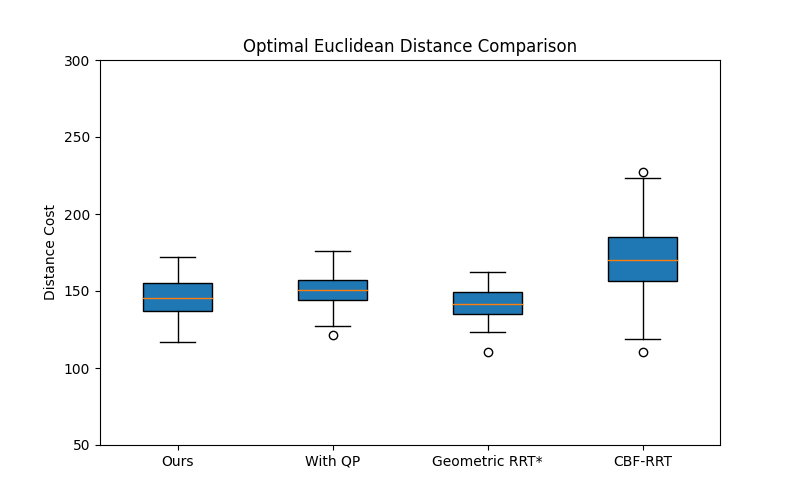}
    \caption{We compare our method w.r.t. other planning methods in terms of the optimal distance cost.}
    \label{fig:benchmark_cost}
\end{figure}
Based on the result, implementing a hash table for optimal gain $K_{\mathrm{LQR}}$ can significantly improve the performance, as it avoids the repeated linearization and recalculation of the gain matrices. The graphical result can be found in Figure \ref{fig:benchmark_time}. 

\begin{table}[htp]
\centering
\caption{Efficiency Comparison (Unit: Seconds)}\label{tb:example3}
\resizebox{\columnwidth}{!}{%
\begin{tabular}{ |c|c|c|c|c|c|c|c| } 
\hline
 Baseline & Seed 0& Seed 20& Seed 42 & Seed 45 &Seed 100 & Mean & Std \\
\hline
  \textbf{Ours} & \textbf{10.04} & \textbf{9.06} & \textbf{9.56} & \textbf{10.27} & \textbf{11.14} & \textbf{10.01} & \textbf{0.70}\\
  \hline
  (2) & 32.07 & 27.8 & 29.41 &34.86 &35.66 &31.96&3.03\\
  \hline
  (3) & 74.34 & 61.6& 109.92 & 104.64& 92.24&88.55&18.20\\
  \hline
  (4) & 41.32 & T/O& 45.57 & 42.18& 43.22&43.07&1.84\\
  \hline
\end{tabular}}
\end{table}

\noindent \textbf{Linear System} For the linear system comparison, our proposed method is about 86\% faster than using LQR-CBF-RRT*-QP (Baseline (5)). The result can be found in Table \ref{tb:qp-vs-noqp}. 
\begin{table}[htp]
\centering
\caption{Efficiency Comparison (Unit: Seconds)}\label{tb:qp-vs-noqp}
\resizebox{\columnwidth}{!}{%
\begin{tabular}{ |c|c|c|c|c|c|c|c| } 
\hline
Baseline & Seed 0& Seed 20& Seed 42 & Seed 45 &Seed 100 & Mean & Std \\
\hline
  \textbf{Ours} & \textbf{11.3} & \textbf{11.15} & \textbf{9.75} & \textbf{10.91} & \textbf{12.42} & \textbf{11.11} & \textbf{0.95}\\
  \hline
  (5) & 80.76 & 79.17 & 78.42& 81.90&80.72 &80.194&1.39\\
  \hline

\end{tabular}}
\end{table}
Our method is compared with an offline MPC with time horizon of 5 and discrete time interval of 0.05. The MPC based method went through 2000 iterations over 41 minutes during the experiment. Finally, we performed an optimal Euclidean distance comparison Fig. \ref{fig:benchmark_cost}. The result shows that LQR-CBF-RRT* does optimize the euclidean distance w.r.t. CBF-RRT.
\vspace{-5 pt}
\subsection{Numerical Example 1: Double Integrator Model}
\label{example1}
In this example, we performed our sampling-based motion planner on a double integrator model with the linear dynamics with the state $[x_1,x_2, x_3, x_4]$ : $\Ddot{x_{1}} = u_{1}$ and $\Ddot{x_{3}} = u_{2}$
% \begin{align} \label{eq:2ddoubleIntSys}
% \left[\begin{matrix}\dot{x_1} \\ \dot{x_2} \\ \dot{x_3} \\ \dot{x_4}  \end{matrix} \right]= \left[ \begin{matrix} 0&1&0&0\\ 0&0&0&0\\0&0&0&1\\0&0&0&0 \end{matrix} \right] \left[\begin{matrix} x_1 \\ x_2\\ x_3\\ x_4 \end{matrix} \right]+ \left[ \begin{matrix} 0&0\\ 1&0 \\0&0 \\0&1 \end{matrix} \right] \left[ \begin{matrix} u_1 \\ u_2\end{matrix} \right],
% \end{align}
where $[x_1,x_3]$ is the position, $[x_2,x_4]$ is the velocity. We control the system using acceleration $[u_1,u_2]$. For the $i$-th obstacle, the corresponding $i$-th safety set can be defined based on the function $h_i(x) = (x_1-x_{1,i,o})^2+(x_3-x_{3,i,o})^2-r_i^2$, where $[x_{1,i,o},x_{3,i,o}]$ is the centroid of the $i$-th obstacle and $r_i$ is the radius. We define CBF constraint $\zeta_i$ for obstacle $i$ as
\begin{align*}
    \zeta_i & = 2x_2^2+2x_4^2+2(x_1 -x_{1,i,o} )u_1 + 2(x_3 -x_{3,i,o} )u_2 \\
            & + k_1 h_i + 2 k_2 [(x_1 -x_{1,i,o})x_2+(x_3 -x_{3,i,o})x_4] \geq 0.
\end{align*}
We perform constraint checking in \texttt{LQR-CBF-Steer} with $\zeta_i$ in both edge extension and rewiring procedures. The satisfaction of CBF constraints $\zeta_i \geq 0 , \forall i$ guarantees the generated controls and state trajectories are safe. We iterate through in total of 2500 steps, and the final result is shown in Figure \ref{Fig2:d}. 

\subsection{Numerical Example 2: Unicycle Model}
In this example, we validate our framework on a unicycle model with system dynamics $\dot{x_1} = v \cos(\theta),\dot{x_2} = v \sin(\theta),\dot{\theta} = \omega $, where $[x_1,x_2]$ is the position, and $\theta$ is the heading angle. The control input $u = [v, \omega]$ consists of the translational and angular velocity. Given this system dynamics, the control inputs $v$, and $\omega$ has mixed relative degree ($v$ has a relative degree 1 and $\omega$ has a relative degree 2), which does not allow us to construct a CBF constraint directly. To bypass this issue, we fix the translational velocity $v$ and only enforce CBF constraints on the angular velocity $\omega$. Next, the CBF constraint for the $i$-th obstacle is constructed as the following
\begin{align*}
    \zeta_i(x) &= 2x_1v^2\cos^2{\theta}+2x_2v^2\sin^2{\theta}\\ \nonumber&+[2(x_2-x_{2,i,o})v\cos{\theta}-2(x_1-x_{1,i,o})v\sin{\theta}]\omega \\ \nonumber &+k_1 h_i(x) + k_2 \pounds_{f} h_i(x) \geq 0, \\ \nonumber
    h_{i}(x) &= (x_1 - x_{1,i,o})^2 + (x_2 - x_{2,i,o})^2-r_i^2 \\ \nonumber
    \pounds_{f} h_{i}(x) &= 2v(x_1 - x_{1,i,o})\cos{\theta} + 2v(x_2 - x_{2,i,o})\sin{\theta}. \\ \nonumber
\end{align*}
The algorithm performs 3000 iterations before termination, and the result can be found in Figure \ref{fig:unicycle_example}.

\subsection{Hardware Experiment}
\begin{figure}
    \centering    \includegraphics[width=0.65\linewidth]{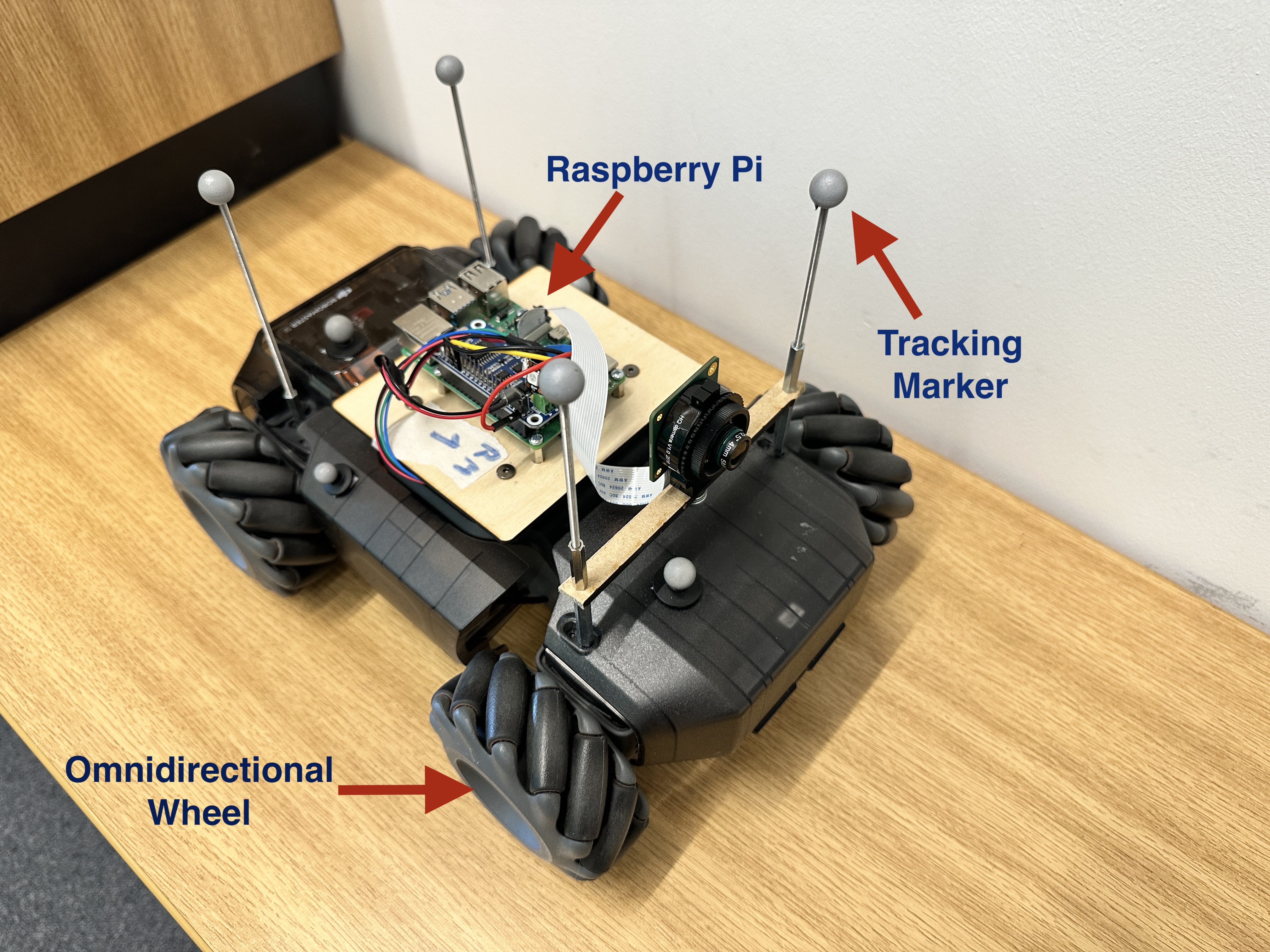}
    \caption{We use a customized DJI Robomaster robot as our experimental platform \cite{9811744}. The robot is equipped with a Raspberry Pi for onboard computation.}
    \label{fig:robomaster}
\end{figure}
\begin{figure}
    \centering    \includegraphics[width=0.65\linewidth]{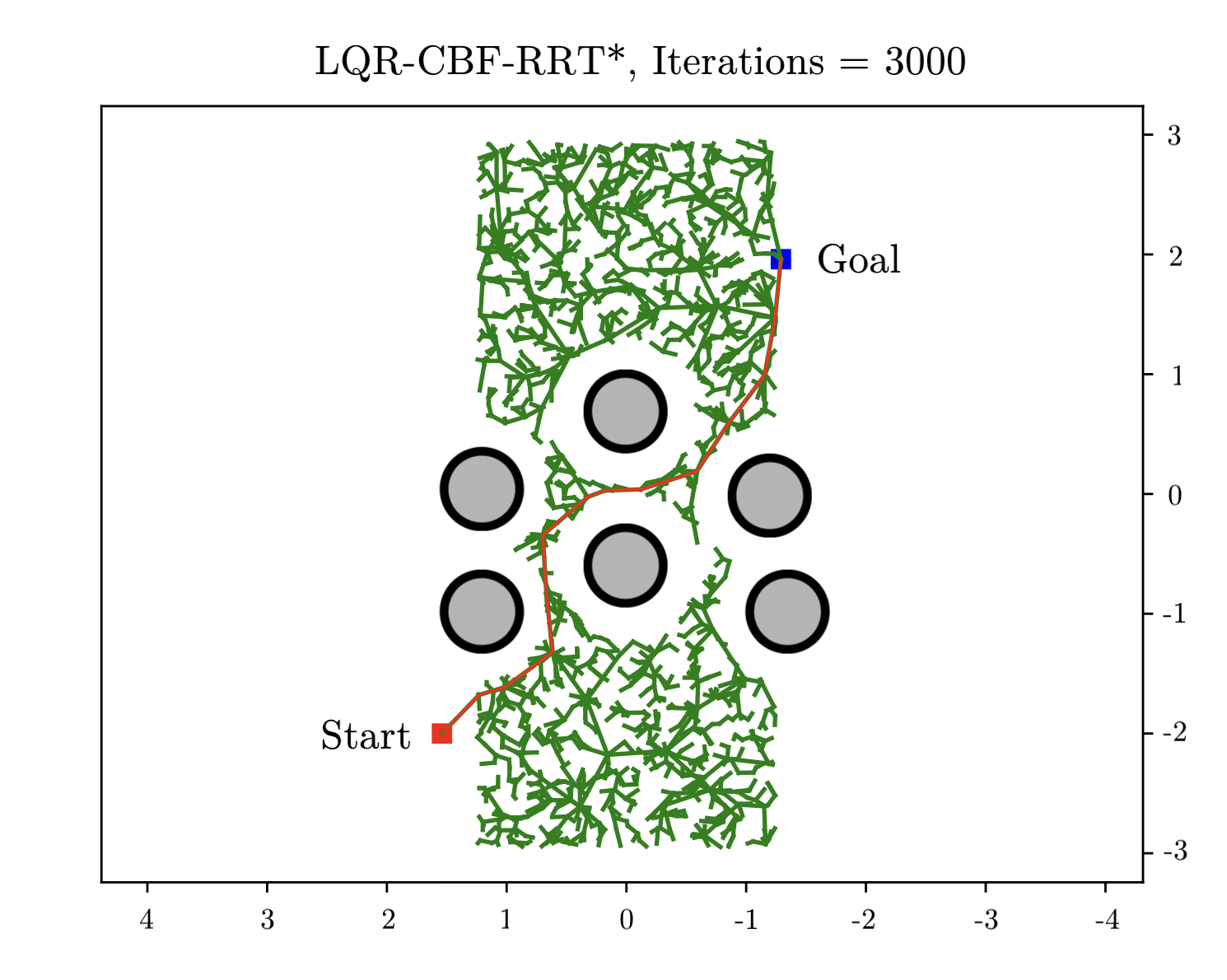}
    \caption{Before deploying the robot, we tested the planner in a simulation. In the visual representation, each green edge denotes a sampled trajectory, while the red trajectory illustrates the final solution obtained after 3,000 iterations.}
    \label{fig:planned_traj}
\end{figure}
 To evaluate the effectiveness of our planner in real-world scenario, we performed an experiment with an omnidirectional robot (Figure \ref{fig:robomaster}) to validate our method. The robot was tasked with successfully navigating through a cluttered environment featuring six obstacles. The obstacles are over-approximated with circular shapes with a radius of 0.25 meters for constructing the CBFs. To accurately track the robot's position, we employed an external telemetry system (OptiTrack), effectively simulating an outdoor GPS-enabled environment. The onboard Raspberry Pi computer is running ROS2 for communicating with the positioning system and processing control commands. A customized software stack handles the robot's low-level control called Frejya \cite{shankar2021freyja}. In this scenario, we first used our planner to generate the optimal trajectories offline (shown in Figure \ref{fig:planned_traj}). Then, the robot used its own online MPC controller to track the generated trajectory plan. The result (Figure \ref{fig:exp_hardware}) showed it is able efficiently navigate to the pre-defined goal without any collisions. The video is available at ~\footnote{\url{https://www.youtube.com/watch?v=yE2yhEQSSUY}}.
\begin{figure}
    \centering    \includegraphics[width=0.7\linewidth]{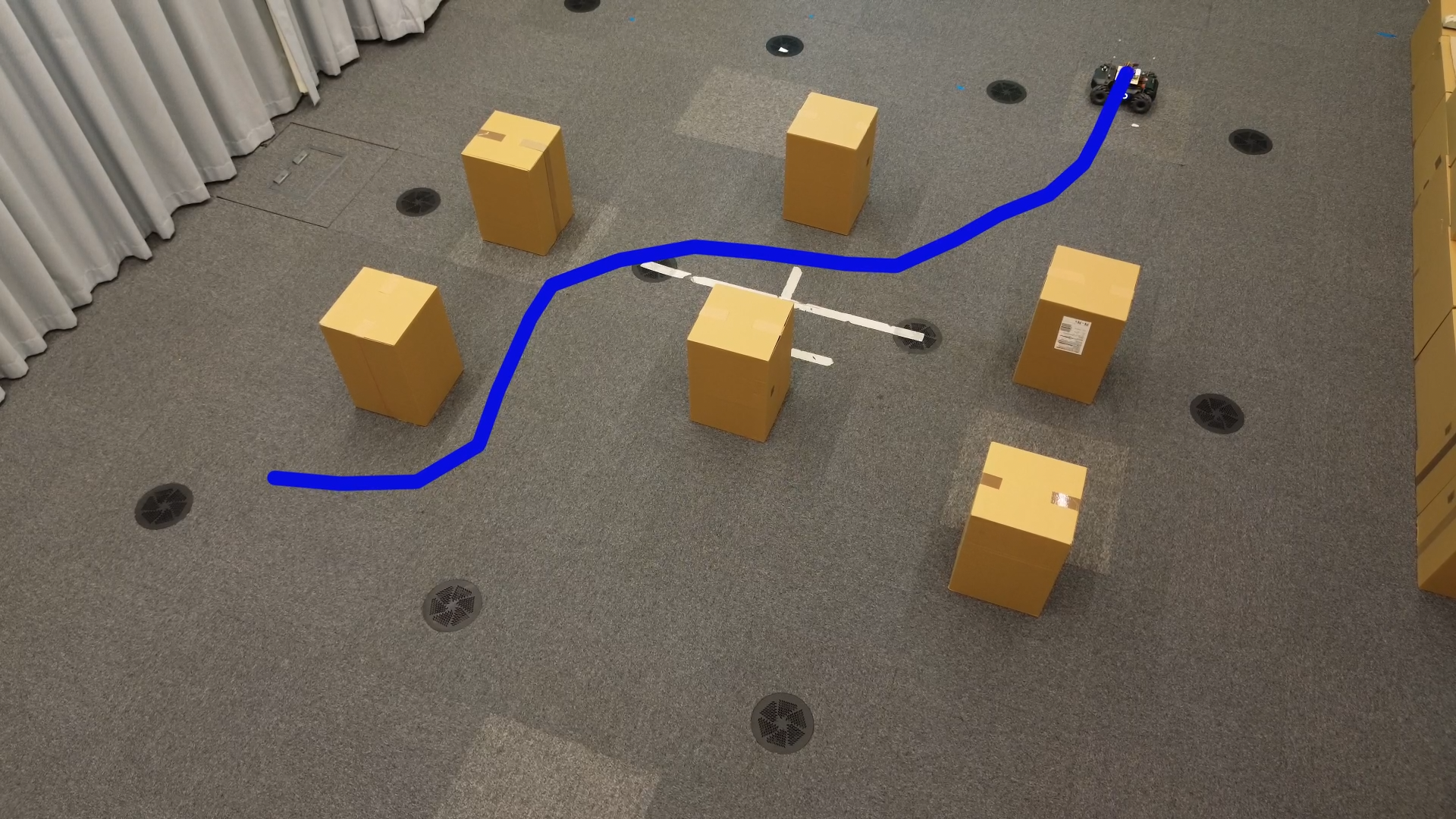}
    \caption{Showcasing an omni-directional ground robot that safely navigates around obstacles using our planner. The path followed by the robot is highlighted in blue.}
    \label{fig:exp_hardware}
\end{figure}

\section{Conclusion}
In this paper, we formulated an offline sampling-based motion planning problem to optimize the LQR cost function and ensure safety. We employ the CEM method, further boosting the efficiency of our algorithm during the sampling stage. Notably, our technique outperforms benchmark counterparts in comparative tests and demonstrates robust results in real-world experiments.
\iffalse
\begin{figure*}
\centering
 \begin{minipage}[t]{1\linewidth}
\subfloat[Iteration 500] {\label{Fig1:a} \includegraphics[width=0.22\linewidth]{Figures/LQR-CBF-RRT Unicycle Iteration 500.png}}\hfill
%\subfloat[Iteration 1500] {\label{Fig2:b} \includegraphics[width=0.31\linewidth]{Figures/LQR-CBF-RRT Iteration 1500.pdf}}\hfill
\subfloat[Iteration 3000] {\label{Fig3:c} \includegraphics[width=0.22\linewidth]{Figures/LQR-CBF-RRT Unicycle Iteration 3000.png}}\hfill
\subfloat[SDF Level Set and Elite Samples at iteration 500] {\label{Fig1:a} \includegraphics[width=0.22\linewidth]{Figures/kde_500_unicycle_elite.png}}\hfill
%\subfloat[SDF Level Set 1500] {\label{Fig2:b} \includegraphics[width=0.31\linewidth]{Figures/kde_1500.pdf}}\hfill
\subfloat[SDF Level Set and Elite Samples at iteration 3000] {\label{Fig3:c} \includegraphics[width=0.22\linewidth]{Figures/kde_3000_unicycle_elite.png}}
\caption{We test our motion planner on the unicycle model over 3000 iterations with adaptive sampling for this example. The two left plots demonstrate the edge extension process and the final state trajectory. The initial state $x_{\mathrm{init}}=[2,2]$ from the bottom left corner and the goal state $x_{\mathrm{goal}}=[30,24]$. The two right plots show the SDF distributions with elite samples.}
\label{Fig2:nonlinear_sys_example}
\end{minipage}
\end{figure*} 
\fi

% \clearpage
\bibliographystyle{IEEEtran}
\bibliography{reference}
\end{document}